\documentclass{article}

\usepackage{microtype}
\usepackage{graphicx}
\usepackage{booktabs} 

\usepackage{amsfonts}       
\usepackage{nicefrac}       
\usepackage{microtype}      
\usepackage{xcolor}         
\usepackage{todonotes}      
\usepackage{tabularx}
\usepackage{tikz-cd}        
\usepackage{subcaption}
\usepackage{bm}
\usepackage[export]{adjustbox}
\usepackage{mathtools}
\usepackage{pifont}
\newcommand{\cmark}{\ding{51}}%
\newcommand{\xmark}{\ding{55}}%

\usepackage{amsthm}
\usepackage{thmtools, thm-restate}
\declaretheorem[numberwithin=section]{thm}
\declaretheorem[sibling=thm]{lemma}

\declaretheorem{assumption}

\declaretheorem[sibling=thm]{definition}
\usepackage{xspace}
\DeclareRobustCommand{\eg}{e.g.,\@\xspace}
\DeclareRobustCommand{\ie}{i.e.,\@\xspace}
\DeclareRobustCommand{\wrt}{w.r.t.\@\xspace}
\usepackage{amsmath}
\usepackage{amssymb}
\usepackage{nicefrac}
\usepackage{dsfont}
\newcommand{\ind}{\mathds{1}}
\newcommand{\de}{\,\mathrm{d}}
\newcommand{\Aspace}{\mathcal{A}}
\newcommand{\Sspace}{\mathcal{S}}
\newcommand{\Hspace}{\mathcal{H}}
\newcommand{\R}{\mathcal{R}}
\newcommand{\uR}{\overline{\mathcal{R}}}
\newcommand{\lR}{\underline{\mathcal{R}}}
\newcommand{\V}{\mathcal{V}}
\newcommand{\cmp}{\mathcal{M}}
\newcommand{\mdp}{\mathcal{M}^R}
\newcommand{\emdp}{\widetilde{\mathcal{M}}^R_T}
\newcommand{\pomdp}{\mathcal{M}^R_\Omega}
\newcommand{\epomdp}{\widetilde{\mathcal{M}}^R_\Omega}
\newcommand{\J}{\mathcal{J}}
\DeclareMathOperator*{\EV}{\mathbb{E}}
\DeclareMathOperator*{\Var}{\mathbb{V}ar}
\newcommand{\Reals}{\mathbb{R}}
\DeclareMathOperator*{\Pinm}{\Pi_{NM}}

\DeclareMathOperator*{\Pim}{\Pi_{M}}
\newcommand{\Pic}[2]{\Pi_{#1}^{#2}}
\DeclareMathOperator*{\argmax}{arg\,max}
\DeclareMathOperator*{\argmin}{arg\,min}
\newcommand{\distset}[2]{\mathcal{D}_{#1}^{#2}}
\DeclareMathOperator*{\NM}{NM}

\DeclareMathOperator*{\M}{M}
\DeclareMathOperator*{\D}{D}
\DeclareMathOperator*{\Pc}{P}
\DeclareMathOperator*{\NP}{NP}
\usepackage{enumitem}
\setitemize{align=left, leftmargin=*, itemsep=2pt, topsep=0pt}

\usepackage{hyperref}

\usepackage[accepted]{icml2022}

\icmltitlerunning{The Importance of Non-Markovianity in Maximum State Entropy Exploration}

\begin{document}

\twocolumn[
\icmltitle{The Importance of Non-Markovianity in Maximum State Entropy Exploration}



\icmlsetsymbol{equal}{*}

\begin{icmlauthorlist}
\icmlauthor{Mirco Mutti}{equal,poli,uni}
\icmlauthor{Riccardo De Santi}{equal,eth}
\icmlauthor{Marcello Restelli}{poli}
\end{icmlauthorlist}

\icmlaffiliation{poli}{Politecnico di Milano}
\icmlaffiliation{uni}{Universit\`a di Bologna}
\icmlaffiliation{eth}{ETH Zurich}

\icmlcorrespondingauthor{Mirco Mutti}{mirco.mutti@polimi.it}
\icmlcorrespondingauthor{Riccardo De Santi}{rdesanti@ethz.ch}

\icmlkeywords{Machine Learning, ICML}

\vskip 0.3in
]



\printAffiliationsAndNotice{\icmlEqualContribution} 

\begin{abstract}
In the maximum state entropy exploration framework, an agent interacts with a reward-free environment to learn a policy that maximizes the entropy of the expected state visitations it is inducing. \citet{hazan2019maxent} noted that the class of Markovian stochastic policies is sufficient for the maximum state entropy objective, and exploiting non-Markovianity is generally considered pointless in this setting.
In this paper, we argue that non-Markovianity is instead paramount for maximum state entropy exploration in a finite-sample regime. Especially, we recast the objective to target the expected entropy of the induced state visitations in a single trial. Then, we show that the class of non-Markovian deterministic policies is sufficient for the introduced objective, while Markovian policies suffer non-zero regret in general. However, we prove that the problem of finding an optimal non-Markovian policy is NP-hard. Despite this negative result, we discuss avenues to address the problem in a tractable way and how non-Markovian exploration could benefit the sample efficiency of online reinforcement learning in future works.
\end{abstract}

\section{Introduction}
Several recent works have addressed \emph{Maximum State Entropy} (MSE) exploration~\cite{hazan2019maxent, tarbouriech2019active, lee2019smm, mutti2020intrinsically, mutti2020policy, zhang2020exploration, guo2021geometric, liu2021behavior, liu2021aps, seo2021state, yarats2021reinforcement, mutti2021unsupervised, nedergaard2022k} as an objective for unsupevised Reinforcement Learning (RL)~\cite{sutton2018reinforcement}. In this line of work, an agent interacts with a reward-free environment~\cite{jin2020reward} in order to maximize an entropic measure of the state distribution induced by its behavior over the environment, effectively targeting a uniform exploration of the state space. Previous works motivated this MSE objective in two main directions. On the one hand, this learning procedure can be seen as a form of \emph{unsupervised pre-training} of the base model~\cite{laskin2021urlb}, which has been extremely successful in supervised learning~\cite{erhan2009difficulty, erhan2010does, brown2020gpt}. In this view, a MSE policy can serve as an exploratory initialization to standard learning techniques, such as Q-learning~\cite{watkins1992q} or policy gradient~\cite{peters2008reinforcement}, and this has been shown to benefit the sample efficiency of a variety of RL tasks that could be specified over the pre-training environment~\citep[\eg][]{mutti2020policy, liu2021behavior, laskin2021urlb}. On the other hand, pursuing a MSE objective leads to an even coverage of the state space, which can be instrumental to address the \emph{sparse reward discovery} problem~\cite{tarbouriech2020gosprl}. Especially, even when the fine-tuning is slow~\cite{campos2021coverage}, the MSE policy might allow to solve hard-exploration tasks that are out of reach of RL from scratch~\cite{mutti2020policy, liu2021behavior}. As we find these premises fascinating, and of general interest to the RL community, we believe it is worth providing a theoretical reconsideration of the MSE problem. Specifically, we aim to study the minimal class of policies that is necessary to optimize a well-posed MSE objective, and the general complexity of the resulting learning problem.

All of the existing works pursuing a MSE objective solely focus on optimizing Markovian exploration strategies, in which each decision is conditioned on the current state of the environment rather than the full history of the visited states. The resulting learning problem is known to be provably efficient in tabular domains~\cite{hazan2019maxent, zhang2020variational}.
Moreover, this choice is common in RL, as it is well-known that an optimal deterministic Markovian strategy maximizes the usual cumulative sum of rewards objective~\cite{puterman2014markov}. Similarly, \citet[Lemma 3.3]{hazan2019maxent} note that the class of Markovian strategies is \emph{sufficient} for the standard MSE objective. A carefully constructed Markovian strategy is able to induce the same state distribution of any history-based (non-Markovian) one by exploiting randomization. Crucially, this result does not hold only for asymptotic state distributions, but also for state distributions that are marginalized over a finite horizon~\cite{puterman2014markov}. 
Hence, there is little incentive to consider more complicated strategies as they are not providing any benefit on the value of the entropy objective.
\begin{figure}[t]
    \centering
    \includegraphics[width=0.45\textwidth]{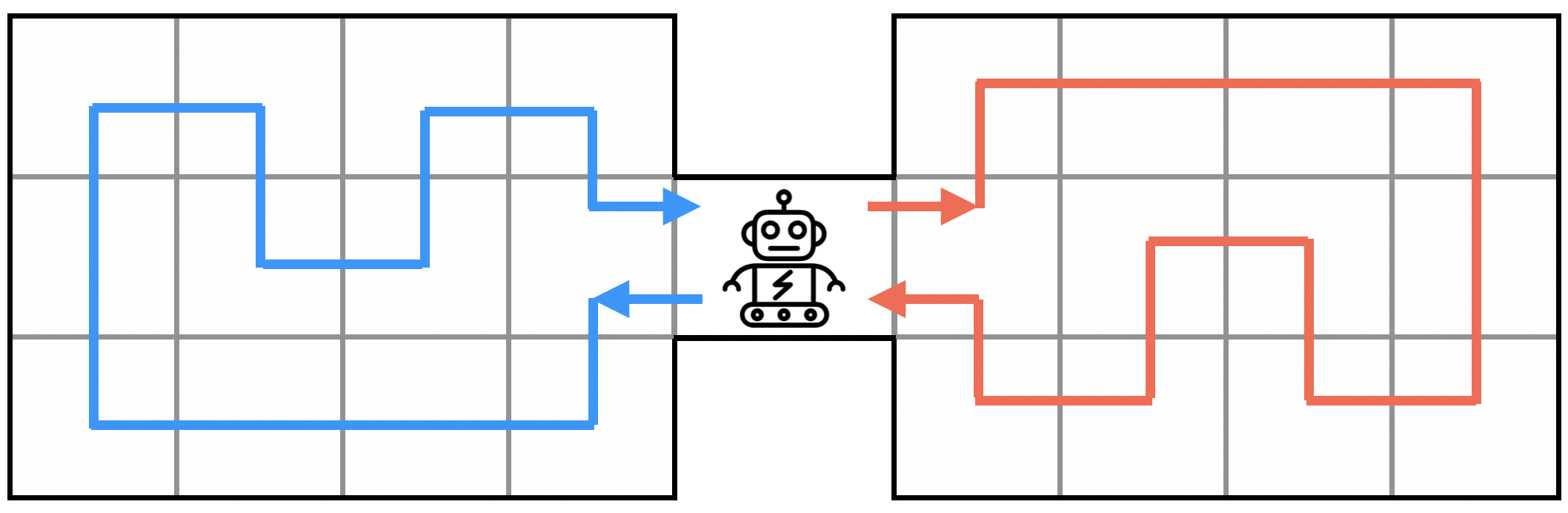}
    \caption{Illustrative two-rooms domain. The agent starts in the middle, colored traces represent optimal strategies to explore the left and the right room.}
    \label{fig:example}
\end{figure}

However, the intuition suggests that exploiting the history of the interactions is useful when the agent's goal is to uniformly explore the environment: If you know what you have visited already, you can take decisions accordingly.
To this point, let us consider an illustrative example in which the agent finds itself in the middle of a two-rooms domain (as depicted in Figure~\ref{fig:example}), having a budget of interactions that is just enough to visit every state within a single episode. It is easy to see that an optimal Markovian strategy for the MSE objective would randomize between going left and right in the initial position, and then would follow the optimal route within a room, finally ending in the initial position again. An episode either results in visiting the left room twice, or the right room twice, or each room once, and all of this outcomes have the same probability. Thus, the agent might explore poorly when considering a single episode, but the exploration is uniform in the average of \emph{infinite trials}. Arguably, this is quite different from how a human being would tackle this problem, \ie taking intentional decisions in the middle position to visit a room before the other. This strategy leads to uniform exploration of the environment in \emph{any trial}, but it is inherently non-Markovian.

Backed by this intuition, we argue that prior work does not recognize the importance of non-Markovianity in MSE exploration due to an hidden infinite-samples assumption in the objective formulation, which is in sharp contrast with the objective function it is actually optimized by empirical methods, \ie the state entropy computed over a finite batch of interactions. In this paper, we introduce a new \emph{finite-sample} MSE objective that is akin to the practical formulation, as it targets the expected entropy of the state visitation frequency induced within an episode instead of the entropy of the expected state visitation frequency over infinite samples. In this finite-sample formulation non-Markovian strategies are crucial, and we believe they can benefit a significant range of relevant applications. For example, collecting task-specific samples might be costly in some real-world domains, and a pre-trained non-Markovian strategy is essential to guarantee quality exploration even in a single-trial setting. In another instance, one might aim to pre-train an exploration strategy for a class of multiple environments instead of a single one. A non-Markovian strategy could exploit the history of interactions to swiftly identify the structure of the environment, then employing the environment-specific optimal strategy thereafter. Unfortunately, learning a non-Markovian strategy is in general much harder than a Markovian one, and we are able to show that it is NP-hard in this setting. Nonetheless, this paper aims to highlight the importance of non-Markovinaity to fulfill the promises of maximum state entropy exploration, thereby motivating the development of tractable formulations of the problem as future work.

The contributions are organized as follows. 
First, in Section~\ref{sec:infinite_samples}, we report a known result~\cite{puterman2014markov} to show that the class of Markovian strategies is sufficient for any infinite-samples MSE objective, including the entropy of the induced marginal state distributions in episodic settings.
Then, in Section~\ref{sec:finite_samples}, we propose a novel finite-sample MSE objective and a corresponding regret formulation. Especially, we prove that the class of non-Markovian strategies is sufficient for the introduced objective, whereas the optimal Markovian strategy suffers a non-zero regret.
However, in Section~\ref{sec:complexity_analysis}, we show that the problem of finding an optimal non-Markovian strategy for the finite-sample MSE objective is NP-hard in general. 
Despite the hardness result, we provide a numerical validation of the theory (Section~\ref{sec:numerical_validation}), and we comment some potential options to address the problem in a tractable way (Section~\ref{sec:conclusions}).
In Appendix~\ref{apx:related_work}, we discuss the related work in the MSE literature, while the missing proofs can be found in Appendix~\ref{apx:proofs}.

\section{Preliminaries}
\label{sec:preliminaries}
In the following, we will denote with $\Delta(\mathcal{X})$ the simplex of a space $\mathcal{X}$, with $[T]$ the set of integers $\{0, \ldots, T - 1 \}$, and with $v \oplus u$ a concatenation of the vectors $v, u$.

\paragraph{Controlled Markov Process}
A Controlled Markov Process (CMP) is a tuple $\cmp := (\Sspace, \Aspace, P, \mu)$, where $\Sspace$ is a finite state space ($|\Sspace| = S$), $\Aspace$ is a finite action space ($|\Aspace| = A$), $P: \Sspace \times \Aspace \to \Delta (\Sspace)$ is the transition model, such that $P(s' | a, s)$ denotes the probability of reaching state $s' \in \Sspace$ when taking action $a \in \Aspace$ in state $s \in \Sspace$, and $\mu \in \Delta (\Sspace)$ is the initial state distribution.

\paragraph{Policies} 
A policy $\pi$ defines the behavior of an agent interacting with an environment modelled by a CMP. It consists of a sequence of decision rules $\pi := (\pi_t)_{t = 0}^\infty$. Each of them is a map between histories $h := (s_{j}, a_{j})_{j = 0}^t \in \Hspace_t$ and actions $\pi_t : \Hspace_t \to \Delta (\mathcal{A})$, such that $\pi_t (a | h)$ defines the conditional probability of taking action $a \in \Aspace$ having experienced the history $h \in \Hspace_t$. We denote as $\Hspace$ the space of the histories of arbitrary length. We denote as $\Pi$ the set of all the policies, and as $\Pi^{\D}$ the set of deterministic policies $\pi = (\pi_t)_{t = 1}^{\infty}$ such that $\pi_t : \Hspace_t \to \Aspace$. We further define:
\begin{itemize}
    \item \emph{Non-Markovian} (NM) policies $\Pinm$, where each $\pi \in \Pinm$ collapses to a single time-invariant decision rule $\pi = (\pi, \pi, \ldots)$ such that $\pi : \Hspace \to \Delta (\mathcal{A})$;
    \item \emph{Markovian} (M) policies $\Pim$, where each $\pi \in \Pim$ is defined through a sequence of Markovian decision rules $\pi = (\pi_t)_{t = 0}^\infty$ such that $\pi_t : \mathcal{S} \to \Delta (\mathcal{A})$. A Markovian policy that collapses into a single time-invariant decision rule $\pi = (\pi, \pi, \ldots)$ is called a \emph{stationary} policy.
\end{itemize}

\paragraph{State Distributions and Visitation Frequency}
A policy $\pi \in \Pi$ interacting with a CMP induces a $t$-step state distribution $d_t^\pi (s) := Pr (s_t = s | \pi)$ over $\Sspace$~\cite{puterman2014markov}. This distribution is described by the temporal relation $d^\pi_t (s) = \int_{\Sspace} \int_{\Aspace} d^\pi_{t-1} (s', a') P (s|s',a') \de s' \de a'$, where $d_t^\pi (\cdot, \cdot) \in \Delta(\Sspace \times \Aspace) $ is the $t$-step state-action distribution.
We call the asymptotic fixed point of this temporal relation the \emph{stationary state distribution} $d_{\infty}^\pi (s) := \lim_{t \to \infty} d_t^{\pi} (s)$, and we denote as $d^\pi_\gamma(s) := (1 - \gamma) \sum^{\infty}_{t=0}\gamma^t d^\pi_t (s)$ its $\gamma$-discounted counterpart, where $\gamma \in (0, 1)$ is the discount factor.
A marginalization of the $t$-step state distribution over a finite horizon $T$, \ie $d_{T}^\pi (s) := \frac{1}{T} \sum_{t \in [T]} d^\pi_t (s)$, is called the \emph{marginal state distribution}. The \emph{state visitation frequency} $d_{h} (s) = \frac{1}{T} \sum_{t \in [T]} \ind (s_t = s | h)$ is a realization of the marginal state distribution, such that $  \EV_{h \sim p^\pi_T} \big[ d_h (s) \big] = d_{T}^\pi (s)$, where the distribution over histories $p^\pi_T \in \Delta (\Hspace_T)$ is defined as $ p^\pi_T (h) = \mu(s_{0}) \prod_{t \in [T - 1]} \pi(a_{t}| h_t) P(s_{t + 1}|a_{t}, s_{t}).$

\paragraph{Markov Decision Process}
A CMP $\cmp$ paired with a reward function $R: \Sspace \times \Aspace \to \Reals$ is called a Markov Decision Process (MDP)~\cite{puterman2014markov} $\mdp := \cmp \cup R$. We denote with $R(s,a)$ the expected immediate reward when taking action $a \in \Aspace$ in $s \in \Sspace$, and with $R (h) = \sum_{t \in [T]} R(s_{t}, a_{t})$ the return over the horizon $T$. The performance of a policy $\pi$ over the MDP $\mdp$ is defined as the \emph{average return} $\mathcal{J}_{\mdp} (\pi) = \EV_{h \sim p^\pi_T} [ R (h)]$, and $\pi_{\mathcal{J}}^* \in \argmax_{\pi \in \Pi} \mathcal{J}_{\mdp} (\pi)$ is called an optimal policy. For any MDP $\mdp$, there always exists a deterministic Markovian policy $\pi \in \Pi_{\M}^{\D}$ that is optimal~\cite{puterman2014markov}.

\paragraph{Extended MDP}
The problem of finding an optimal non-Markovian policy with history-length $T$ in an MDP $\mdp$, \ie  $\pi_{\NM}^* \in \argmax_{\pi \in \Pinm} \J_{\mdp} (\pi)$, can be reformulated as the one of finding an optimal Markovian policy $\pi_{\M}^* \in \argmax_{\pi \in \Pim} \J_{\emdp} (\pi)$ in an extended MDP $\emdp$. The extended MDP is defined as $\emdp := (\widetilde{\Sspace}, \widetilde{\Aspace}, \widetilde{P}, \widetilde{R}, \widetilde{\mu})$, in which $\widetilde{\Sspace} \subseteq \Hspace_{[T]} = \Hspace_{1} \cup \ldots \cup \Hspace_{T}$, and $\widetilde{s} := (\widetilde{s}_0, \ldots, \widetilde{s}_{-1} )$ corresponds to a history in $\mdp$ of length $|\widetilde{s}|$, $\widetilde{\Aspace} = \Aspace$, $\widetilde{P} (\widetilde{s}' | \widetilde{s}, \widetilde{a}) = P (s' = \widetilde{s}'_{-1} | s = \widetilde{s}_{-1}, a = \widetilde{a})$, $\widetilde{R} (\widetilde{s}, \widetilde{a}) = R(s = \widetilde{s}_{-1}, a = \widetilde{a})$, and $\widetilde{\mu} (\widetilde{s}) = \mu (s = \widetilde{s})$ for any $\widetilde{s} \in \widetilde{\Sspace}$ of unit length.

\paragraph{Partially Observable MDP}
A Partially Observable Markov Decision Process (POMDP)~\cite{astrom1965optimal, kaelbling1998planning} is described by $\pomdp := (\Sspace, \Aspace, P, R, \mu, \Omega, O)$, where $\Sspace, \Aspace, P, R, \mu$ are defined as in an MDP, $\Omega$ is a finite observation space, and $O : \Sspace \times \Aspace \to \Delta (\Omega)$ is the observation function, such that $O (o | s', a)$ denotes the conditional probability of the observation $o \in \Omega$ when selecting action $a \in \Aspace$ in state $s \in \Sspace$. Crucially, while interacting with a POMDP the agent cannot observe the state $s \in \Sspace$, but just the observation $o \in \Omega$. The performance of a policy $\pi$ is defined as in an MDP.

\section{Infinite Samples: Non-Markovianity Does Not Matter}
\label{sec:infinite_samples}
Previous works pursuing maximum state entropy exploration of a CMP consider an objective of the kind
\begin{equation}
	\mathcal{E}_{\infty} (\pi) := H \big( d^\pi (\cdot) \big) = - \EV_{s \sim d^\pi} \big[ \log d^\pi (s)\big ],
	\label{eq:infinite_samples_entropy}
\end{equation}
where $d^\pi (\cdot)$ is either a stationary state distribution~\cite{mutti2020intrinsically}, a discounted state distribution~\cite{hazan2019maxent, tarbouriech2019active}, or a marginal state distribution~ \cite{lee2019smm, mutti2020policy}. 
While it is well-known~\cite{puterman2014markov} that there exists an optimal deterministic policy $\pi^* \in \Pic{\M}{\D}$ for the common average return objective $\mathcal{J}_{\mdp}$, it is not pointless to wonder whether the objective in~\eqref{eq:infinite_samples_entropy} requires a more powerful policy class than $\Pim$. \citet[][Lemma 3.3]{hazan2019maxent} confirm that the set of (randomized) Markovian policies $\Pim$ is indeed sufficient for $\mathcal{E}_{\infty}$ defined over asymptotic (stationary or discounted) state distributions. 
In the following theorem and corollary, we report a common MDP result~\cite{puterman2014markov} to show that $\Pim$ suffices for $\mathcal{E}_{\infty}$ defined over (non-asymptotic) marginal state distributions as well.
\begin{restatable}[]{theorem}{distributionsEquivalence}
	Let $x \in \{\infty, \gamma, T \}$, and let $\distset{\NM}{x} = \{ d_x^\pi (\cdot) : \pi \in \Pinm \}$, $\distset{\M}{x} = \{ d_x^\pi (\cdot) : \pi \in \Pim \}$ the corresponding sets of state distributions over a CMP. We can prove that:
	\begin{itemize}
		\item[(i)] The sets of stationary state distributions are equivalent $\distset{\NM}{\infty} \equiv \distset{\M}{\infty}$;
		\item[(ii)] The sets of discounted state distributions are equivalent $\distset{\NM}{\gamma} \equiv \distset{\M}{\gamma}$ for any $\gamma$;
		\item[(iii)] The sets of marginal state distributions are equivalent $\distset{\NM}{T} \equiv \distset{\M}{T}$ for any $T$.
	\end{itemize}
    \label{thr:distributions_equivalence}
\end{restatable}
\begin{proof}[Proof Sketch]
	For any non-Markovian policy $\pi \in \Pinm$ inducing distributions $d_t^\pi (\cdot), d_t^\pi (\cdot, \cdot)$ over the states and the state-action pairs of the CMP, we can build a Markovian policy $\pi' \in \Pim, \pi' = (\pi'_t)_{t = 0}^{\infty}$ through the construction $\pi'_t (a | s) = d_t^\pi (s, a) \big/ d_t^\pi (s), \forall s \in \Sspace, \forall a \in \Aspace$. From~\citep[][Theorem 5.5.1]{puterman2014markov} we know that $d_t^\pi (s) = d_t^{\pi'} (s)$ holds for any $t \geq 0$ and $\forall s \in \Sspace$. This implies that $d_{\infty}^{\pi} (\cdot) = d_{\infty}^{\pi'} (\cdot)$, $d_{\gamma}^{\pi} (\cdot) = d_{\gamma}^{\pi'} (\cdot)$, $d_{T}^{\pi} (\cdot) = d_{T}^{\pi'} (\cdot)$, from which $\distset{\NM}{x} \equiv \distset{\M}{x}$ follows. See Appendix~\ref{apx:proofs_infinite_samples} for a detailed proof.
\end{proof}
From the equivalence of the sets of induced distributions, it is straightforward to derive the optimality of Markovian policies for objective~\eqref{eq:infinite_samples_entropy}.
\begin{restatable}[]{corollary}{optimalMarkovDistributions}
	For every CMP, there exists a Markovian policy $\pi^* \in \Pim$ such that $\pi^* \in \argmax_{\pi \in \Pi} \mathcal{E}_{\infty} (\pi)$. 
	\label{thr:optimal_Markov_distributions}
\end{restatable}
As a consequence of Corollary~\ref{thr:optimal_Markov_distributions}, there is little incentive to consider non-Markovian policies when optimizing objective~\eqref{eq:infinite_samples_entropy}, since there is no clear advantage to make up for the additional complexity of the policy.
This result might be unsurprising when considering asymptotic distributions, as one can expect a carefully constructed Markovian policy to be able to tie the distribution induced by a non-Markovian policy in the limit of the interaction steps.
However, it is less evident that a similar property holds for the expectation of final-length interactions alike.
Yet, we were able to show that a Markovian policy that properly exploits randomization can always achieve equivalent state distributions \wrt non-Markovian counterparts. 
Note that state distributions are actually \emph{expected} state visitation frequency, and the expectation practically implies an infinite number of realizations.
In this paper, we show that this underlying infinite-sample regime is the reason why the benefit of non-Markovianity, albeit backed up by intuition, does not matter. Instead, we propose a relevant finite-sample entropy objective in which non-Markovianity is crucial.

\section{Finite Samples: Non-Markovianity Matters}
\label{sec:finite_samples}
In this section, we reformulate the typical maximum state entropy exploration objective of a CMP~\eqref{eq:infinite_samples_entropy} to account for a finite-sample regime. Crucially, we consider the expected entropy of the state visitation frequency rather than the entropy of the expected state visitation frequency, which results in
\begin{equation}
	\mathcal{E} (\pi) := \EV_{h \sim p^\pi_T} \big[ H \big( d_h (\cdot) \big) \big] = - \EV_{h \sim p^\pi_T} \EV_{s \sim d_h} \big[ \log d_h (s) \big].
	\label{eq:finite_samples_entropy}
\end{equation}
We note that $\mathcal{E} (\pi) \leq \mathcal{E}_{\infty} (\pi)$ for any $\pi \in \Pi$, which is trivial by the concavity of the entropy function and the Jensen's inequality.
Whereas~\eqref{eq:finite_samples_entropy} is ultimately an expectation as it is~\eqref{eq:infinite_samples_entropy}, the entropy is not computed over the infinite-sample state distribution $d^\pi_T (\cdot)$ but its finite-sample realization $d_h (\cdot)$. Thus, to maximize $\mathcal{E} (\pi)$ we have to find a policy inducing high-entropy state visits within a single trajectory rather than high-entropy state visits over infinitely many trajectories.
Crucially, while Markovian policies are as powerful as any other policy class in terms of induced state distributions (Theorem~\ref{thr:distributions_equivalence}), this is no longer true when looking at induced trajectory distributions $p^\pi_T$. Indeed, we will show that non-Markovianity provides a superior policy class for objective~\eqref{eq:finite_samples_entropy}. First, we define a performance measure to formally assess this benefit, which we call the \emph{regret-to-go}.\footnote{Note that the entropy function does not enjoy additivity, thus we cannot adopt the usual expected cumulative regret formulation in this setting.}
\begin{definition}[Expected Regret-to-go]
	Consider a policy $\pi \in \Pi$ interacting with a CMP over $T - t$ steps starting from the trajectory $h_t$. We define the expected regret-to-go $\R_{T - t}$, \ie from step $t$ onwards, as
	\begin{equation*}
    		\R_{T - t} (\pi, h_t) = H^* - \EV_{h_{T - t} \sim p^\pi_{T - t}} \big[ H \big( d_{h_t \oplus h_{T - t}} (\cdot) \big) \big],
	\end{equation*}
	where $H^* = \max_{\pi^* \in \Pi} \EV_{h_{T - t}^* \sim p^{\pi^*}_{T-t}} \big[ H \big( d_{h_t \oplus h_{T - t}^*} (\cdot) \big) \big]$ is the expected entropy achieved by an optimal policy $\pi^*$. The term $R_T (\pi)$ denotes the expected regret-to-go of a $T$-step trajectory $h_T$ starting from $s \sim \mu$.
	\label{thr:regret-to-go}
\end{definition}
The intuition behind the regret-to-go is quite simple. Suppose to have drawn a trajectory $h_t$ upon step $t$. If we take the subsequent action with the (possibly sub-optimal) policy $\pi$, by how much would we decrease (in expectation) the entropy of the state visits $H(d_{h_T} (\cdot))$ \wrt an optimal policy $\pi^*$?
In particular, we would like to know how limiting the policy $\pi$ to a specific policy class would affect the expected regret-to-go and the value of $\mathcal{E} (\pi)$ we could achieve. 
The following theorem and subsequent corollary, which constitute the main contribution of this paper, state that an optimal non-Markovian policy suffers zero expected regret-to-go in any case, whereas an optimal Markovian policy suffers non-zero expected regret-to-go in general.
\begin{restatable}[Non-Markovian Optimality]{theorem}{regret}
	For every CMP $\cmp$ and trajectory $h_t \in \Hspace_{[T]}$, there exists a deterministic non-Markovian policy $\pi_{\NM} \in \Pic{\NM}{\D}$ that suffers zero regret-to-go $\R_{T - t} (\pi_{\NM}, h_t) = 0$, whereas for any $\pi_{\M} \in \Pim$ we have $\R_{T - t} (\pi_{\M}, h_t) \geq 0$.
	\label{thr:regret_theorem}
\end{restatable}
\begin{restatable}[Sufficient Condition]{corollary}{sufficientCondition}
	For every CMP $\cmp$ and trajectory $h_t \in \Hspace_{[T]}$ for which any optimal Markovian policy $\pi_{\M} \in \Pim$ is randomized (\ie stochastic) in $s_t$, we have strictly positive regret-to-go $\R_{T - t} (\pi_{\M}, h_t) > 0$.
	\label{thr:sufficient_condition}
\end{restatable}
The result of Theorem~\ref{thr:regret_theorem} highlights the importance of non-Markovianity for optimizing the finite-sample MSE objective~\eqref{eq:finite_samples_entropy}, as the class of Markovian policies is dominated by the class of non-Markovian policies. Most importantly, Corollary~\ref{thr:sufficient_condition} shows that non-Markovian policies are strictly better than Markovian policies in any CMP of practical interest, \ie those in which any optimal Markovian policy has to be randomized~\cite{hazan2019maxent} in order to maximize~\eqref{eq:finite_samples_entropy}. The intuition behind this result is that a Markovian policy would randomize to make up for the uncertainty over the history, whereas a non-Markovian policy does not suffer from this partial observability, and it can deterministically select an optimal action. Clearly, this partial observability is harmless when dealing with the standard RL objective, in which the reward is fully Markovian and does not depend on the history, but it is instead relevant in the peculiar MSE setting, in which the objective is a concave function of the state visitation frequency. In the following section, we report a sketch of the derivation underlying Theorem~\ref{thr:regret_theorem} and Corollary~\ref{thr:sufficient_condition}, while we refer to the Appendix~\ref{apx:proofs_finite_samples} for complete proofs.

\subsection{Regret Analysis}

To the purpose of the regret analysis, we will consider the following assumption to ease the notation.\footnote{Note that this assumption could be easily removed by partitioning the action space in $h_t$ as $\mathcal{A} (h_t) = \mathcal{A}_{opt} (h_t) \cup \mathcal{A}_{sub-opt} (h_t)$, such that $\mathcal{A}_{opt} (h_t)$ are optimal actions and $\mathcal{A}_{sub-opt} (h_t)$ are sub-optimal, and substituting any term $\pi (a^* | h_t)$ with $\sum_{a \in \mathcal{A}_{opt} (h_t)} \pi (a | h_t)$ in the results.} 
\begin{assumption}[Unique Optimal Action]
	For every CMP $\cmp$ and trajectory $h_{t} \in \Hspace_{[T]}$, there exists a unique optimal action $a^*\in \Aspace$ \wrt the objective~\eqref{eq:finite_samples_entropy}.
	\label{ass:unique_action}
\end{assumption}
First, we show that the class of deterministic non-Markovian policies is sufficient for the minimization of the regret-to-go, and thus for the maximization of~\eqref{eq:finite_samples_entropy}.
\begin{restatable}[]{lemma}{optimalDeterministic}
	For every CMP $\mathcal{M}$ and trajectory $h_t \in \Hspace_{[T]}$, there exists a deterministic non-Markovian policy $\pi_{\NM} \in \Pic{\NM}{\D}$ such that $\pi_{\NM} \in \argmax_{\pi \in \Pinm} \mathcal{E} (\pi)$, which suffers zero regret-to-go $\R_{T - t} (\pi_{\NM}, h_t) = 0$.
    \label{thr:optimal_deterministic}
\end{restatable}
\begin{proof}
	The result $\R_{T - t} (\pi_{\NM}, h_t) = 0$ is straightforward by noting that the set of non-Markovian policies $\Pinm$ with arbitrary history-length is as powerful as the general set of policies $\Pi$. To show that there exists a deterministic $\pi_{\NM}$, we consider the extended MDP $\emdp$ obtained from the CMP $\cmp$ as in Section~\ref{sec:preliminaries}, in which the extended reward function is $\widetilde{R} (\widetilde{s}, \widetilde{a}) = H( d_{\widetilde{s}} (\cdot) )$ for every $\widetilde{a} \in \widetilde{\Aspace}$, and every $\widetilde{s} \in \widetilde{\Sspace}$ such that $|\widetilde{s}| = T$, and $\widetilde{R} (\widetilde{s}, \widetilde{a}) = 0 $ otherwise. Since a Markovian policy $\widetilde{\pi}_{\M} \in \Pic{\M}{\D}$ on $\emdp$ can be mapped to a non-Markovian policy $\pi_{\NM} \in \Pic{\NM}{\D}$ on $\cmp$, and it is well-known~\cite{puterman2014markov} that for any MDP there exists an optimal deterministic Markovian policy, we have that $\widetilde{\pi}_{\M} \in \argmax_{\pi \in \Pim} \mathcal{J}_{\emdp} (\pi) $ implies $\pi_{\NM} \in \argmax_{\pi \in \Pinm} \mathcal{E} (\pi)$.
\end{proof}
Then, in order to prove that the class of non-Markovian policies is also necessary for regret minimization, it is worth showing that Markovian policies can instead rely on randomization to optimize objective~\eqref{eq:finite_samples_entropy}.
\begin{restatable}[]{lemma}{lawTotalVariance}
	Let $\pi_{\NM} \in \Pic{\NM}{\D}$ a non-Markovian policy such that $\pi_{\NM} \in \argmax_{\pi \in \Pi} \mathcal{E} (\pi)$ on a CMP $\cmp$. For a fixed history $h_t \in \Hspace_t$ ending in state $s$, the variance of the event of an optimal Markovian policy $\pi_{\M} \in \argmax_{\pi \in \Pim} \mathcal{E} (\pi)$ taking $a^* = \pi_{\NM} (h_t)$ in $s$ is given by
	\begin{equation*}
		\Var \big[ \mathcal{B} ( \pi_{\M} (a^*|s, t) ) \big] = \Var_{hs \sim p^{\pi_{\NM}}_t} \big[ \EV \big[ \mathcal{B} ( \pi_{\NM} (a^* | hs) ) \big] \big],
	\end{equation*}
	where $hs \in \Hspace_{t}$ is any history of length $t$ such that the final state is $s$, \ie $hs := (h_{t - 1} \in \Hspace_{t - 1}) \oplus s$, and $\mathcal{B} (x)$ is a Bernoulli with parameter $x$.
    \label{thr:law_total_variance}
\end{restatable}
\begin{proof}[Proof Sketch]
	We can prove the result through the Law of Total Variance (LoTV)~\citep[see][]{bertsekas2002introduction}, which gives
	\begin{align*} 
		\Var \big[ \mathcal{B} ( \pi_{\M} (a^*|s, t) ) \big] = \EV_{hs \sim p^{\pi_{\NM}}_t} \big[ \Var \big[ \mathcal{B} ( \pi_{\NM} (a^* | hs) ) \big] \big] \\ + \Var_{hs \sim p^{\pi_{\NM}}_t} \big[ \EV \big[ \mathcal{B} ( \pi_{\NM} (a^* | hs) ) \big] \big], \qquad \forall s \in \Sspace.
	\end{align*}
	Then, exploiting the determinism of $\pi_{\NM}$ (through Lemma~\ref{thr:optimal_deterministic}), it is straightforward to see that $\EV_{hs \sim p^{\pi_{\NM}}_t} \big[ \Var \big[ \mathcal{B} ( \pi_{\NM} (a^* | hs) ) \big] \big] = 0$, which concludes the proof.\footnote{Note that the determinism of $\pi_{\NM}$ does not also imply $\Var_{hs \sim p^{\pi_{\NM}}_t} \big[ \EV \big[ \mathcal{B} ( \pi_{\NM} (a^* | hs) ) \big] \big] = 0$, as the optimal action $\overline{a} = \pi_{\NM} (hs)$ may vary for different histories, which results in the inner expectations $\EV \big[ \mathcal{B} ( \pi_{\NM} (a^* | hs) ) \big]$ being either $1$ or $0$. }
\end{proof}
Unsurprisingly, Lemma~\ref{thr:law_total_variance} shows that, whenever the optimal strategy for~\eqref{eq:finite_samples_entropy} (\ie the non-Markovian $\pi_{\NM}$) requires to adapt its decision in a state $s$ according to the history that led to it ($hs$), an optimal Markovian policy for the same objective (\ie $\pi_{\M}$) must necessarily be randomized.
This is crucial to prove the following result, which establishes lower and upper bounds $\lR_{T - t}, \uR_{T - t}$ to the expected regret-to-go of any Markovian policy that optimizes~\eqref{eq:finite_samples_entropy}.
\begin{restatable}[]{lemma}{regretBounds}
    \label{thr:regret_bounds}
	Let $\pi_{\M}$ be an optimal Markovian policy $\pi_{\M} \in \argmax_{\pi \in \Pim} \mathcal{E} (\pi)$ on a CMP $\cmp$. For any $h_t \in \Hspace_{[T]}$, it holds $\lR_{T - t} (\pi_{\M}) \leq \R_{T - t} (\pi_{\M}) \leq \uR_{T - t} (\pi_{\M})$ such that
	\begin{align*}
		\lR_{T - t} (\pi_{\M}) &= \frac{ H^* - H^*_2  }{ \pi_{\M} (a^* | s_t)} \Var_{hs_t \sim p^{\pi_{\NM}}_t} \big[ \EV \big[ \mathcal{B} ( \pi_{\NM} (a^* | hs_t) ) \big]  \big], \\
		\uR_{T - t} (\pi_{\M}) &= \frac{ H^* - H_*  }{ \pi_{\M} (a^* | s_t)} \Var_{hs_t \sim p^{\pi_{\NM}}_t} \big[ \EV \big[ \mathcal{B} ( \pi_{\NM} (a^* | hs_t) ) \big]  \big],
	\end{align*}
	where $\pi_{\NM} \in \argmax_{\pi \in \Pic{\NM}{\D}} \mathcal{E} (\pi)$, and $H_*, H^*_2$ are given by
	\begin{align*}
		&H_* = \min_{h \in \Hspace_{T - t}} H(d_{h_t \oplus h} (\cdot)), \\
		&H^*_2 = \max_{h \in \Hspace_{T - t} \setminus \Hspace_{T - t}^*} H(d_{h_t \oplus h} (\cdot)) \\
		&s. t. \; \Hspace_{T - t}^* = \argmax_{h \in \Hspace_{T - t}} H(d_{h_t \oplus h} (\cdot)).
	\end{align*}
\end{restatable}
\begin{proof}[Proof Sketch]
	The crucial idea to derive lower and upper bounds to the regret-to-go is to consider the impact of a sub-optimal action in the best-case and the worst-case CMP respectively (see Lemma~\ref{thr:best_case_cmp},~\ref{thr:worst_case_cmp}). This gives $\R_{T - t} (\pi_{\M}) \geq H^* - \pi_{\M} (a^* | s_t) H^* - \big( 1 - \pi_{\M} (a^* | s_t) \big) H_2^*$ and $\R_{T - t} (\pi_{\M}) \leq H^* - \pi_{\M} (a^* | s_t) H^* - \big( 1 - \pi_{\M} (a^* | s_t) \big) H_*$. Then, with Lemma~\ref{thr:law_total_variance} we get $\Var \big[ \mathcal{B} ( \pi_{\M} (a^*|s_t) ) \big] = \pi_{\M} (a^* | s_t) \big( 1 - \pi_{\M} (a^* | s_t) \big) = \Var_{hs \sim p^{\pi_{\NM}}_t} \big[ \EV \big[ \mathcal{B} ( \pi_{\NM} (a^* | hs_t) ) \big] \big]$, which concludes the proof.
\end{proof}
Finally, the result in Theorem~\ref{thr:regret_theorem} is a direct consequence of Lemma~\ref{thr:regret_bounds}. Note that the upper and lower bounds on the regret-to-go are strictly positive whenever $\pi_{\M} (a^* | s_t) < 1$, as it is stated in Corollary~\ref{thr:sufficient_condition}.

\section{Complexity Analysis}
\label{sec:complexity_analysis}
Having established the importance of non-Markovianity in dealing with MSE exploration in a finite-sample regime, it is worth considering how hard it is to optimize the objective~\ref{eq:finite_samples_entropy} within the class of non-Markovian policies. Especially, we aim at characterizing the complexity of the problem:
\begin{equation*}
	\Psi_0 := \underset{\pi \in \Pinm}{\text{maximize}} \; \mathcal{E} (\pi),
\end{equation*}
defined over a CMP $\cmp$. Before going into the details of the analysis, we provide a couple of useful definitions for the remainder of the section, whereas we leave to~\cite{arora2009complexity} an extended review of complexity theory.
\begin{definition}[Many-to-one Reductions]
	We denote as $A \le_m B$ a many-to-one reduction from $A$ to $B$.
\end{definition}
\begin{definition}[Polynomial Reductions]
	We denote as $A \le_p B$ a polynomial-time (Turing) reduction from $A$ to $B$.
\end{definition}
Then, we recall that $\Psi_0$ can be rewritten as the problem of finding a reward-maximizing Markovian policy, \ie $\widetilde{\pi}_{\M} \in \argmax_{\pi \in \Pim} \mathcal{J}_{\emdp} (\pi)$, over a convenient extended MDP $\emdp$ obtained from CMP $\cmp$ (see the proof of Lemma~\ref{thr:optimal_deterministic} for further details). We call this problem $\widetilde{\Psi}_0$ and we note that $\widetilde{\Psi}_0 \in \Pc$, as the problem of finding a reward-maximizing Markovian policy is well-known to be in $\Pc$ for any MDP~\cite{papadimitriou1987complexity}. However, the following lemma shows that it does not exist a many-to-one reduction from $\Psi_0$ to $\widetilde{\Psi}_0$.
\begin{restatable}[]{lemma}{noReduction}
    A reduction $\Psi_0 \le_m \widetilde{\Psi}_0$ does not exist.
    \label{thr:no_reduction}
\end{restatable}
\begin{proof}
    In general, coding any instance of $\Psi_0$ in the representation required by $\widetilde{\Psi}_0$, which is an extended MDP $\emdp$, holds exponential complexity w.r.t. the input of the initial instance of $\Psi_0$, \ie a CMP $\cmp$. Indeed, to build the extended MDP $\emdp$ from $\cmp$, we need to define the transition probabilities $\widetilde{P} (\widetilde{s}' | \widetilde{s}, \widetilde{a} )$ for every $\widetilde{s}' \in \widetilde{\mathcal{S}}, \widetilde{a} \in \widetilde{\mathcal{A}}, \widetilde{s} \in \widetilde{\mathcal{S}}$. Whereas the action space remains unchanged $\widetilde{\mathcal{A}} = \mathcal{A}$, the extended state space $\widetilde{\mathcal{S}}$ has cardinality $|\widetilde{\mathcal{S}}| = S^T$ in general, which grows exponentially in $T$.
\end{proof}
The latter result informally suggests that $\Psi_0 \notin \Pc$. 
Indeed, we can now prove the main theorem of this section, which shows that $\Psi_0$ is $\NP$-hard under the common assumption that $\Pc \neq \NP$.
\begin{restatable}[]{theorem}{npHardness}
    $\Psi_0 $ is $\NP$-hard.
    \label{thr:np_hardness}
\end{restatable}
\begin{proof}[Proof Sketch]
To prove the theorem, it is sufficient to show that there exists a problem $\Psi_c \in \NP$-hard so that $\Psi_c \leq_p \Psi_0$. We show this by reducing 3SAT, which is a well-known $\NP$-complete problem, to $\Psi_0$. To derive the reduction we consider two intermediate problems, namely $\Psi_1$ and $\Psi_2$. Especially, we aim to show that the following chain of reductions holds
    \begin{equation*}
        \Psi_0 \geq_m \Psi_1 \geq_p \Psi_2 \geq_p \text{3SAT}.
    \end{equation*}
First, we define $\Psi_1$ and we prove that $\Psi_0 \geq_m \Psi_1$. Informally, $\Psi_1$ is the problem of finding a reward-maximizing Markovian policy $\pi_{\M} \in \Pim$ w.r.t. the entropy objective~\eqref{eq:finite_samples_entropy} encoded through a reward function in a convenient POMDP $\epomdp$. We can build $\epomdp$ from the CMP $\cmp$ similarly as the extended MDP $\emdp$ (see Section~\ref{sec:preliminaries} and the proof of Lemma~\ref{thr:optimal_deterministic} for details), except that the agent only access the observation space $\widetilde{\Omega}$ instead of the extended state space $\widetilde{\Sspace}$. In particular, we define $\widetilde{\Omega} = \Sspace$ (note that $\Sspace$ is the state space of the original CMP $\cmp$), and $\widetilde{O} (\widetilde{o} | \widetilde{s}) = \widetilde{s}_{-1}$.
Then, the reduction $\Psi_0 \geq_m \Psi_1$ works as follows. We denote as $\mathcal{I}_{\Psi_i}$ the set of possible instances of problem $\Psi_i$. We show that $\Psi_0$ is harder than $\Psi_1$ by defining the polynomial-time functions $\psi$ and $\phi$ such that any instance of $\Psi_1$ can be rewritten through $\psi$ as an instance of $\Psi_0$, and a solution $\pi_{\NM}^* \in \Pinm$ for $\Psi_0$ can be converted through $\phi$ into a solution $\pi_{\M}^* \in \Pim$ for the original instance of $\Psi_1$.
\begin{figure*}[ht!]
	\begin{subfigure}[t]{0.34\textwidth}
    		\centering
    		\vspace{0.5cm}
    		\includegraphics[scale=1, valign=t]{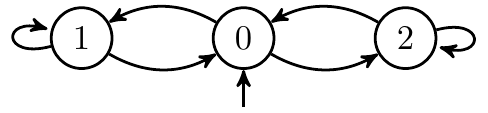}
    		\caption{\emph{3State}}
    		\label{fig:three_state}
    	\end{subfigure}
    	\begin{subfigure}[t]{0.3\textwidth}
    		\centering
    		\includegraphics[scale=1, valign=t]{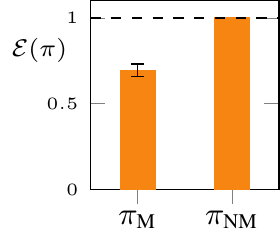}
    		\vspace{-0.05cm}
    		\caption{Average Entropy}
    		\label{fig:three_state_entropy}
    	\end{subfigure}
     \begin{subfigure}[t]{0.3\textwidth}
    		\centering
    		\includegraphics[scale=1, valign=t]{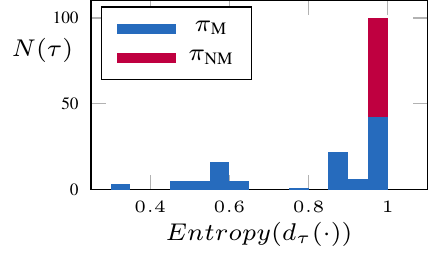}
    		\vspace{-0.27cm}
    		\caption{Entropy Frequency}
    		\label{fig:three_state_histogram}
    	\end{subfigure}
    	
    	\begin{subfigure}[t]{0.34\textwidth}
    		\centering
    		\vspace{0.15cm}
    		\includegraphics[scale=1, valign=t]{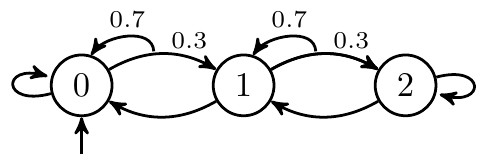}
    		\caption{\emph{River Swim}}
    		\label{fig:river_swim}
    	\end{subfigure}
    	\begin{subfigure}[t]{0.3\textwidth}
    		\centering
    		\includegraphics[scale=1, valign=t]{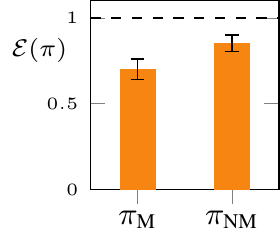}
    		\vspace{-0.07cm}
    		\caption{Average Entropy}
    		\label{fig:river_swim_entropy}
    	\end{subfigure}
     \begin{subfigure}[t]{0.3\textwidth}
    		\centering
    		\includegraphics[scale=1, valign=t]{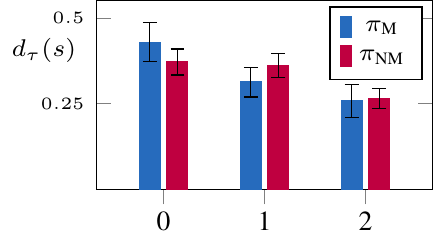}
    		\vspace{-0.15cm}
    		\caption{State Visitation Frequency}
    		\label{fig:river_swim_visits}
    	\end{subfigure}
    \caption{In \textbf{(a, d)}, we illustrates the \emph{3State}  and \emph{River Swim} CMPs. Then, we report the average entropy induced by an optimal (stationary) Markovian policy $\pi_{\M}$ and an optimal non-Markovian policy $\pi_{\NM}$ in the \emph{3State} ($T = 9$) \textbf{(b)} and the \emph{River Swim} ($T =10$) \textbf{(e)}. In \textbf{(c)} we report the entropy frequency in the \emph{3State}, in \textbf{(f)} the state visitation frequency in the \emph{River Swim}. We provide 95\% c.i. over 100 runs.}
    \label{fig:numerical_validation}
\end{figure*}
The function $\psi$ sets $\Sspace = \widetilde{\Omega}$ and derives the transition model of $\cmp$ from the one of $\epomdp$, while $\phi$ converts the optimal solution of $\Psi_0$ by computing $\pi_{\M}^* (a | o, t) = \sum_{ho \in \Hspace_o} p^{\pi_{\NM}^{*}}_T (ho) \pi_{\NM}^* (a|ho) $, where $\Hspace_o$ stands for the set of histories $h \in \Hspace_{t}$ ending in the observation $o \in \Omega$.
Thus, we have that $\Psi_0 \geq_m \Psi_1$ holds. We now define $\Psi_2$ as the policy existence problem w.r.t. the problem statement of $\Psi_1$. Hence, $\Psi_2$ is the problem of determining whether the value of a reward-maximizing Markovian policy $\pi_{\M}^* \in \argmax_{\pi \in \Pim} \mathcal{J}_{\epomdp} (\pi)$ is greater than 0. Since computing an optimal policy in POMDPs is in general harder than the relative policy existence problem ~\citep[][Section 3]{lusena2001complexity}, we have that $\Psi_1 \geq_p \Psi_2$.
For the last reduction, \ie $\Psi_2 \geq_p \text{3SAT}$, we extend the proof of Theorem 4.13 in \cite{mundhenk2000complexity}, which states that the policy existence problem for POMDPs is $\NP$-complete. In particular, we show that this holds within the restricted class of POMDPs defined in $\Psi_1$.
Since the chain $\Psi_0 \geq_m \Psi_1 \geq_p \Psi_2 \geq_p \text{3SAT}$ holds, we have that $\Psi_0 \geq_p \text{3SAT}$. Since $ \text{3SAT} \in \NP$-complete, we can conclude that $\Psi_0$ is $\NP$-hard.
\end{proof}
Having established the hardness of the optimization of $\Psi_0$, one could now question whether the problem $\Psi_0$ is instead easy to verify ($\Psi_0 \in \NP$), from which we would conclude that $\Psi_0 \in \NP$-complete. Whereas we doubt that this problem is significantly easier to verify than to optimize, the focus of this work is on its optimization version, and we thus leave as future work a finer analysis to show that $\Psi_0 \notin \NP$.

\section{Numerical Validation}
\label{sec:numerical_validation}
Despite the hardness result of Theorem~\ref{thr:np_hardness}, we provide a brief numerical validation around the potential of non-Markovianity in MSE exploration. Crucially, the reported analysis is limited to simple domains and short time horizons, and it has to be intended as an illustration of the theoretical claims reported in previous sections. For the sake of simplicity, in this analysis we consider stationary policies for the Markovian set, though similar results can be obtained for time-variant strategies as well (in stochastic environments).
Whereas a comprehensive evaluation of the practical benefits of non-Markovianity in MSE exploration is left as future work, we discuss in Section~\ref{sec:conclusions} why we believe that the development of scalable methods is not hopeless even in this challenging setting.

In this section, we consider a \emph{3State} ($S = 3, A = 2, T = 9$), which is a simple abstraction of the two-rooms in Figure~\ref{fig:example}, and a \emph{River Swim}~\cite{strehl2008analysis} ($S = 3, A = 2, T = 10$) that are depicted in Figure~\ref{fig:three_state}, \ref{fig:river_swim} respectively. Especially, we compare the expected entropy~\eqref{eq:finite_samples_entropy} achieved by an optimal non-Markovian policy $\pi_{\NM} \in \argmax_{\pi \in \Pi_{\NM}} \mathcal{E} (\pi)$, which is obtained by solving the extended MDP as described in the proof of Lemma~\ref{thr:optimal_deterministic}, against an optimal Markovian policy $\pi_{\M} \in \argmax_{\pi \in \Pi_{\M}} \mathcal{E} (\pi)$. 
In confirmation of the result in Theorem~\ref{thr:regret_theorem}, $\pi_{\M}$ cannot match the performance of $\pi_{\NM}$ (see Figure~\ref{fig:three_state_entropy}, \ref{fig:river_swim_entropy}).
In \emph{3State}, an optimal strategy requires going left when arriving in state $0$ from state $2$ and vice versa. The policy $\pi_{\NM}$ is able to do that, and it always realizes the optimal trajectory (Figure~\ref{fig:three_state_histogram}). Instead, $\pi_{\M}$ is uniform in $0$ and it often runs into sub-optimal trajectories.
In the \emph{River Swim}, the main hurdle is to reach state $2$ from the initial one. Whereas $\pi_{\M}$ and $\pi_{\NM}$ are equivalently good in doing so, as reported in Figure~\ref{fig:river_swim_visits}, only the non-Markovian strategy is able to balance the visitations in the previous states when it eventually reaches $2$. The difference is already noticeable with a short horizon and it would further increase with a longer $T$.

\section{Discussion and Conclusion}
\label{sec:conclusions}
In the previous sections, we detailed the importance of non-Markovianity when optimizing a finite-sample MSE objective, but we also proved that the corresponding optimization problem is NP-hard in its general formulation.
Despite the hardness result, we believe that it is not hopeless to learn exploration policies with some form of non-Markovianity, while still preserving an edge over Markovian strategies. In the following paragraphs, we discuss potential avenues to derive practical methods for relevant relaxations to the general class of non-Markovian policies.

\textbf{Finite-Length Histories}~~~Throughout the paper, we considered non-Markovian policies that condition their decisions on histories of arbitrary length, \ie $\pi : \Hspace \to \Delta(\Aspace)$. However, the complexity of optimizing such policies grows exponentially with the length of the history. To avoid this exponential blowup, one can define a class of non-Markovian policies $\pi : \Hspace_H \to \Delta(\Aspace)$ in which the decisions are conditioned on histories of a finite length $H > 1$ that are obtained from a sliding window on the full history. The optimal policy within this class would still retain better regret guarantees than an optimal Markovian policy, but it would not achieve zero regret in general. With the length parameter $H$ one can trade-off the learning complexity with the regret according to the structure of the domain. For instance, $H = 2$ would be sufficient to achieve zero regret in the \emph{3State} domain, whereas in the \emph{River Swim} domain any $H < T$ would cause some positive regret.

\textbf{Compact Representations of the History}~~~Instead of setting a finite length $H$, one can choose to perform function approximation on the full history to obtain a class of policies $\pi : f(\Hspace) \to \Delta (\Aspace)$, where $f$ is a function that maps an history $h$ to some compact representation. An interesting option is to use the notion of \emph{eligibility traces}~\cite{sutton2018reinforcement} to encode the information of $h$ in a vector of length $S$, which is updated as $\bm{z}_{t + 1} \gets \lambda \bm{z}_t + \bm{1}_{s_t}$, where $\lambda \in (0, 1)$ is a discount factor, $\bm{1}_{s_t}$ is a vector with a unit entry at the index $s_t$, and $\bm{z}_0 = 0$. The discount factor $\lambda$ acts as a smoothed version of the length parameter $H$, and it can be dynamically adapted while learning. Indeed, this eligibility traces representation is particularly convenient for policy optimization~\cite{deisenroth2013survey}, in which we could optimize in turn a parametric policy over actions $\pi_{\bm{\theta}} (\cdot | \bm{z}, \lambda)$ and a parametric policy over the discount $\pi_{\bm{\nu}} (\lambda)$. To avoid a direct dependence on $S$, one can define the vector $\bm{z}$ over a discretization of the state space.

\textbf{Deep Recurrent Policies}~~~Another noteworthy way to do function approximation on the history is to employ recurrent neural networks~\cite{williams1989learning, hochreiter1997lstm} to represent the non-Markovian policy. This kind of recurrent architecture is already popular in RL. In this paper we are providing the theoretical ground to motivate the use of deep recurrent policies to address maximum state entropy exploration.

\textbf{Non-Markovian Control with Tree Search}~~~In principle, one can get a realization of actions from the optimal non-Markovian policy without ever computing it, \eg by employing a Monte-Carlo Tree Search (MCTS)~\cite{kocsis2006bandit} approach to select the next action to take. Given the current state $s_t$ as a root, we can build the tree of trajectories from the root through repeated simulations of potential action sequences. With a sufficient number of simulations and a sufficiently deep tree, we are guaranteed to select the optimal action at the root. If the horizon is too long, we can still cut the tree at any depth and approximately evaluate a leaf node with the entropy induced by the path from the root to the leaf. The drawback of this procedure is that we require to access a simulator with reset (or a reliable estimate of the transition model) to actually build the tree.

Having reported interesting directions to learn non-Markovian exploration policies in practice, we would like to mention some relevant online RL settings that might benefit from such exploration policies. We leave as future work a formal definition of the settings and an empirical study.

\textbf{Single-Trial RL}~~~In many relevant real-world scenarios, where data collection might be costly or non-episodic in nature, we cannot afford multiple trials to achieve the desired exploration of the environment. Non-Markovian exploration policies guarantee a good coverage of the environment in a single trial and they are particularly suitable for online learning processes.

\textbf{Learning in Latent MDPs}~~~In a latent MDP scenario~\cite{hallak2015contextual, kwon2021latent} an agent interacts with an (unknown) environment drawn from a class of MDPs to solve an online RL task. A non-Markovian exploration policy pre-trained on the whole class could exploit the memory to perform a fast identification of the specific context that has been drawn, quickly adapting to the optimal environment-specific policy.

In this paper we focus on the gap between non-Markovian and Markovian policies, which can be either stationary or time-variant. Future works might consider the role of stationarity~\citep[see also][]{akshay2013steady, laroche2022non}, such as establishing under which conditions stationary strategies are sufficient in this setting. Finally, here we focus on state distributions, which is most common in the MSE literature, but similar results could be extended to state-action distributions with minor modifications.

To conclude, we believe that this work sheds some light on the, previously neglected, importance of non-Markovianity to address maximum state entropy exploration. Although it brings a negative result about the computational complexity of the problem, we believe it can provide inspiration for future empirical and theoretical contributions on the matter.

%

\section*{Acknowledgements}
We would like to thank the reviewers of this paper for their feedbacks and useful advices. We also thank Romain Laroche and Remi Tachet des Combes for having signalled a technical error in a previous draft of the manuscript.

\bibliography{biblio}
\bibliographystyle{icml2022}

\clearpage
\onecolumn
\appendix

\section{Related Work}
\label{apx:related_work}
\citet{hazan2019maxent} were the first to consider an entropic measure over the state distribution as a sensible learning objective for an agent interacting with a reward-free environment~\cite{jin2020reward}. Especially, they propose an algorithm, called MaxEnt, that learns a mixture of policies that collectively maximize the Shannon entropy of the discounted state distribution, \ie~\eqref{eq:infinite_samples_entropy}. The final mixture is learned through a conditional gradient method, in which the algorithm iteratively estimates the state distribution of the current mixture to define an intrinsic reward function, and then identifies the next policy to be added by solving a specific RL sub-problem with this reward. A similar methodology has been obtained by~\citet{lee2019smm} from a game-theoretic perspective on the MSE exploration problem. Their algorithm, called SMM, targets the Shannon entropy of the marginal state distribution instead of the discounted distribution of MaxEnt. Another approach based on the conditional gradient method is FW-AME~\cite{tarbouriech2019active}, which learns a mixture of policies to maximize the entropy of the stationary state-action distribution. As noted in~\cite{tarbouriech2019active}, the mixture of policies might suffer a slow mixing to the asymptotic distribution for which the entropy is maximized. In~\cite{mutti2020intrinsically}, the authors present a method (IDE$^3$AL) to learn a single exploration policy that simultaneously accounts for the entropy of the stationary state-action distribution and the mixing time.

Even if they are sometimes evaluated on continuous domains (especially~\cite{hazan2019maxent, lee2019smm}), the methods we mentioned require an accurate estimate of either the state distribution~\cite{hazan2019maxent, lee2019smm} or the transition model~\cite{tarbouriech2019active, mutti2020intrinsically}, which hardly scales to high-dimensional domains. A subsequent work by~\citet{mutti2020policy} proposes an approach to estimate the entropy of the state distribution through a non-parametric method, and then to directly optimize the estimated entropy via policy optimization. Their algorithm, called MEPOL, is able to learn a single exploration policy that maximizes the entropy of the marginal state distribution in challenging continuous control domains. \citet{liu2021behavior} combine non-parametric entropy estimation with learned state representations into an algorithm, called APT, that successfully addresses MSE exploration problems in visual-inputs domains. \citet{seo2021state} shows that even random state representations are sufficient to learn MSE exploration policies from visual inputs. On a similar line, \citet{yarats2021reinforcement} consider simultaneously learning state representations and a basis for the latent space (or prototypical representations) to help reducing the variance of the entropy estimates. Finally, \citet{liu2021aps} consider a method, called APS, to learn a set of code-conditioned policies that collectively maximizes the MSE objective by coupling non-parametric entropy estimation and successor representation.

Whereas all of the previous approaches accounts for the Shannon entropy in their objectives, recent works~\cite{zhang2020exploration, guo2021geometric} consider alternative formulations. \citet{zhang2020exploration} argues that the R\'enyi entropy provides a superior incentive to cover all of the corresponding space than the Shannon entropy, and they propose a method to optimize the R\'enyi of the state-action distribution via gradient ascent (MaxR\'enyi). On an orthogonal direction, the authors of~\cite{guo2021geometric} consider a reformulation of the entropy function that accounts for the underlying geometry of the space. They present a method, called GEM, to learn an optimal policy for the geometry-aware entropy objective.

\begin{table}[t!]
\caption{Overview of the methods addressing MSE exploration in a controlled Markov process. For each method, we report the nature of the corresponding MSE objective, \ie the entropy function (Entropy), whether it considers stationary, discounted, or marginal distributions (Distribution), and if it accounts for the state space $\Sspace$ or the state-action space $\Sspace\Aspace$ (Space). We also specify if the method learns a single policy rather than a mixture of policies (Mixture), and if it supports non-parametric entropy estimation (Non-Parametric).}
\centering
\begin{tabular}{llllcc}
	\toprule
	Algorithm & Entropy & Distribution & Space & Mixture & Non-Parametric \\ 
	\midrule
	MaxEnt~\cite{hazan2019maxent} & Shannon & discounted & state & \cmark & \xmark \\
	FW-AME~\cite{tarbouriech2019active} & Shannon & stationary & state-action & \cmark & \xmark \\
	SMM~\cite{lee2019smm} & Shannon & marginal & state & \cmark & \xmark \\
	IDE$^3$AL~\cite{mutti2020intrinsically} & Shannon & stationary & state-action & \xmark & \xmark \\
	MEPOL~\cite{mutti2020policy} & Shannon & marginal & state & \xmark & \cmark \\
	MaxR\'enyi~\cite{zhang2020exploration} & R\'enyi & discounted & state-action & \xmark & \xmark \\
	GEM~\cite{guo2021geometric} & geometry-aware & marginal & state & \xmark & \xmark \\
	APT~\cite{liu2021behavior} & Shannon & marginal & state & \xmark & \cmark \\
	RE3~\cite{seo2021state} & Shannon & marginal & state & \xmark & \cmark \\
	Proto-RL~\cite{yarats2021reinforcement} & Shannon & marginal & state & \xmark & \cmark \\
	APS~\cite{liu2021aps} & Shannon & marginal & state & \cmark & \cmark \\
	\bottomrule
\end{tabular}
\end{table}

\subsection{Online Learning of Global Concave Rewards}

Another interesting pair of related works~\cite{cheung2019exploration, cheung2019regret} addresses a reinforcement learning problem for the maximization of concave functions of vectorial rewards, which in a special case~\citep[][Section 6.2]{cheung2019exploration} is akin to our objective function~\eqref{eq:finite_samples_entropy}.
Beyond this similarity in the objective definition, those works and our paper differ for some crucial aspects, and the contributions are essentially non-overlapping. On the one hand, \cite{cheung2019exploration, cheung2019regret} deal with an online learning problem, in which they care for the performance of the policies deployed during the learning process, whereas we only consider the performance of the optimal policy. On the other hand, we aim to compare the classes of non-Markovian and Markovian policies respectively, whereas they consider non-stationary or adaptive strategies to maximize the online objective. Finally, their definition of regret is based on an infinite-samples relaxation of the problem, whereas we account for the performance of the optimal general policy \wrt the finite-sample objective~\eqref{eq:finite_samples_entropy} to define our regret-to-go (Definition~\ref{thr:regret-to-go}).

\section{Missing Proofs}
\label{apx:proofs}
\subsection{Proofs of Section~\ref{sec:infinite_samples}}
\label{apx:proofs_infinite_samples}

\distributionsEquivalence*

\begin{proof}
	First, note that a non-Markovian policy $\pi \in \Pinm$ can always reduce to a Markovian policy $\pi \in \Pim$ by conditioning the decision rules on the history length. Thus, $\distset{\NM}{x} \supseteq \distset{\M}{x}$ is straightforward for any $x \in \{ \infty, \gamma, T \}$. 
From the derivations in~\citep[][Theorem 5.5.1]{puterman2014markov}, we have that $\distset{\M}{x} \supseteq \distset{\NM}{x}$ as well. Indeed, for any non-Markovian policy $\pi \in \Pinm$, we can build a (non-stationary) Markovian policy $\pi' \in \Pim$ as
	\begin{equation*}
	\begin{aligned}
        &\pi' = \big( \pi'_1, \pi'_2, \ldots, \pi'_t, \ldots \big), 
        &\text{such that } \pi'_t (a | s) = \frac{d_t^\pi (s, a)}{d_t^\pi (s)},\;\; \forall s \in \Sspace, \forall a \in \Aspace.
    \end{aligned}
    \end{equation*}
    For $t = 0$, we have that $d^\pi_0 (\cdot) = d^{\pi'}_0 (\cdot)= \mu (\cdot)$, which is the initial state distribution. We proceed by induction to show that if $d^\pi_{t-1} (\cdot) = d^{\pi'}_{t-1} (\cdot)$, then we have
    \begin{align*}
        d_t^{\pi'}(s)
        &= \sum_{s' \in \Sspace} \sum_{a \in \Aspace} d_{t-1}^{\pi'} (s') \pi'_{t-1} (a | s') P(s | s', a) \\
        &= \sum_{s' \in \Sspace} \sum_{a \in \Aspace}  \frac{d_{t-1}^{\pi'} (s') }{d^\pi_{t-1} (s')} d^\pi_{t-1} (s', a) P(s | s', a) \\
        &= \sum_{s' \in \Sspace} \sum_{a \in \Aspace}  d^\pi_{t-1} (s', a) P(s | s', a) \\
        &= d_t^\pi (s).
    \end{align*}
    Since $d_t^\pi (s) = d_t^{\pi'} (s)$ holds for any $t \geq 0$ and $\forall s \in \Sspace$, we have  $d_{\infty}^{\pi} (\cdot) = d_{\infty}^{\pi'} (\cdot)$, $d_{\gamma}^{\pi} (\cdot) = d_{\gamma}^{\pi'} (\cdot)$, $d_{T}^{\pi} (\cdot) = d_{T}^{\pi'} (\cdot)$, and thus $\distset{\M}{x} \supseteq \distset{\NM}{x}$. Then, $\distset{\NM}{x} \equiv \distset{\M}{x}$ follows.
\end{proof}

\optimalMarkovDistributions*

\begin{proof}
	The result is straightforward from Theorem~\ref{thr:distributions_equivalence} and noting that the set of non-Markovian policies $\Pinm$ with arbitrary history-length is as powerful as the general set of policies $\Pi$. Thus, for every policy $\pi \in \Pi$ there exists a (possibly randomized) policy $\pi' \in \Pim$ inducing the same (stationary, discounted or marginal) state distribution of $\pi$, \ie $d^\pi (\cdot) = d^{\pi'} (\cdot)$, which implies $H \big( d^\pi (\cdot) \big) = H \big( d^{\pi'} (\cdot) \big)$. If it holds for any $\pi \in \Pi$, then it holds for $\pi^* \in \argmax_{\pi \in \Pi} H \big( d^\pi (\cdot) \big)$.
\end{proof}

\subsection{Proofs of Section~\ref{sec:finite_samples}}
\label{apx:proofs_finite_samples}

\regret*
\begin{proof}
	The result $\R_{T - t} (\pi_{\NM}, h_t) = 0$ for a policy $\pi_{\NM} \in \Pic{\D}{\NM}$ is a direct implication of Lemma~\ref{thr:optimal_deterministic}, whereas $\R_{T - t} (\pi_{\M}, h_t) \geq 0$ for any $\pi_{\M} \in \Pim$ is given by Lemma~\ref{thr:regret_bounds}, which states that even an optimal Markovian policy $\pi_{\M}^* \in \argmax_{\pi \in \Pim} \mathcal{E} (\pi)$ suffers expected regret-to-go $\R_{T - t} (\pi_{\M}^*) \geq 0$.
\end{proof}

\sufficientCondition*
\begin{proof}
	This result is a direct consequence of the combination of Lemma~\ref{thr:law_total_variance} and Lemma~\ref{thr:regret_bounds}. Indeed, if the policy $\pi_{\M} \in \argmax_{\pi \in \Pim} \mathcal{E} (\pi)$ is randomized in $s_t$ we have 
	\begin{equation*}
		0 < \Var \big[ \mathcal{B} ( \pi_{\M} (a^*|s_t) ) \big] = \Var_{hs_t \sim p^{\pi_{\NM}}_t} \big[ \EV \big[ \mathcal{B} ( \pi_{\NM} (a^* | hs_t) ) \big] \big],
	\end{equation*}
	from Lemma~\ref{thr:law_total_variance}, which gives a lower bound to the expected regret-to-go $\lR_{T - t} (\pi_{\M}, h_t) > 0$ through Lemma~\ref{thr:regret_bounds}.
\end{proof}

\lawTotalVariance*
\begin{proof}
	Let us consider the random variable $A \sim \mathcal{P}$ denoting the event ``the agent takes action $a^* \in \Aspace$''. Through the law of total variance~\cite{bertsekas2002introduction}, we can write the variance of $A$ given $s \in \Sspace$ and $t \geq 0$ as
	\begin{align}
		\Var \big[ A | s, t \big]
		&= \EV \big[ A^2 | s, t \big] - \EV \big[ A | s, t \big]^2 \nonumber \\
		&= \EV_{h} \Big[ \EV \big[ A^2 | s, t, h \big] \Big] - \EV_{h} \Big[ \EV \big[ A | s, t, h \big] \Big]^2 \nonumber \\
		&= \EV_{h} \Big[ \Var \big[ A | s, t, h \big] + \EV \big[ A | s, t, h \big]^2 \Big]
		- \EV_{h} \Big[ \EV_{\pi} \big[ A | s, t, h \big] \Big]^2 \nonumber \\
		&= \EV_{h} \Big[ \Var \big[ A | s, t, h \big] \Big] +\EV_{h} \Big[  \EV \big[ A | s, t, h \big]^2 \Big] - \EV_{h} \Big[ \EV \big[ A | s, t, h \big] \Big]^2 \nonumber \\
		&= \EV_{h} \Big[ \Var \big[ A | s, t, h \big] \Big] + \Var_{h} \Big[ \EV\big[ A | s, t, h \big] \Big]. \label{eq:lotv1}
	\end{align}
	Now let the conditioning event $h$ be distributed as $h \sim p^{\pi_{\NM}}_{t - 1}$, so that the condition $s, t, h$ becomes $hs$ where $hs = (s_{0}, a_{0}, s_{1}, \ldots, s_{t} = s) \in \Hspace_{t}$, and let the variable $A$ be distributed according to $\mathcal{P}$ that maximizes the objective~\eqref{eq:finite_samples_entropy} given the conditioning. Hence, we have that the variable $A$ on the left hand side of~\eqref{eq:lotv1} is distributed as a Bernoulli $\mathcal{B} (\pi_{\M} (a^* | s, t))$, where $\pi_{\M} \in \argmax_{\pi \in \Pim} \mathcal{E} (\pi)$, and the variable $A$ on the right hand side of~\eqref{eq:lotv2} is distributed as a Bernoulli $\mathcal{B} (\pi_{\NM} (a^* | hs))$, where $\pi_{\NM} \in \argmax_{\pi \in \Pinm} \mathcal{E} (\pi)$. Thus, we obtain
	\begin{equation}
		\Var \big[ \mathcal{B} ( \pi_{\M} (a^*|s, t) ) \big] = \EV_{hs \sim p^{\pi_{\NM}}_t} \big[ \Var \big[ \mathcal{B} ( \pi_{\NM} (a^* | hs) ) \big] \big] + \Var_{hs \sim p^{\pi_{\NM}}_t} \big[ \EV \big[ \mathcal{B} ( \pi_{\NM} (a^* | hs) ) \big] \big].
	\label{eq:lotv2}
	\end{equation}
	Under Assumption~\ref{ass:unique_action}, we know from Lemma~\ref{thr:optimal_deterministic} that the policy $\pi_{\NM}$ is deterministic, \ie $\pi_{\NM} \in \Pic{\D}{\NM}$, so that $\Var \big[ \mathcal{B} ( \pi_{\NM} (a^* | hs) ) \big] = 0$ for every $hs$, which concludes the proof.
\end{proof}

\regretBounds*
\begin{proof}
	From the definition of the expected regret-to-go (Definition~\ref{thr:regret-to-go}), we have that
	\begin{equation*}
		\R_{T - t} (\pi_{\M}, h_t) = H^* - \EV_{h_{T - t} \sim p^{\pi_{\M}}_{T - t}} \big[ H \big( d_{h_t \oplus h_{T - t}} (\cdot) \big) \big],
	\end{equation*}
	in which we will omit $h_t$ in the regret-to-go $\R_{T - t} (\pi_{\M}, h_t) = \R_{T - t} (\pi_{\M})$ as $h_t$ is fixed by the statement. To derive a lower bound and an upper bound to $\R_{T - t} (\pi_{\M})$ we consider the impact that taking a sub-optimal action $a \in \Aspace \setminus \{ a^* \}$ in state $s_t$ would have in a best-case and a worst-case CMP respectively, which is detailed in Lemma~\ref{thr:best_case_cmp} and Lemma~\ref{thr:worst_case_cmp}. Especially, we can write
\begin{align*}
	\R_{T - t} (\pi_{\M}) 
	&= H^* - \EV_{h_{T - t} \sim p^{\pi_{\M}}_{T - t}} \big[ H \big( d_{h_t \oplus h_{T - t}} (\cdot) \big) \big] \\
	&\geq H^* - \pi_{\M} (a^* | s_t) H^* - \big( 1 - \pi_{\M} (a^* | s_t) \big) H_{2}^* \\
	&= ( H^*- H^*_2 ) \big( 1 - \pi_{\M} (a^* | s_t) \big)
\end{align*}
and
\begin{align*}
	\R_{T - t} (\pi_{\M}) 
	&= H^* - \EV_{h_{T - t} \sim p^{\pi_{\M}}_{T - t}} \big[ H \big( d_{h_t \oplus h_{T - t}} (\cdot) \big) \big] \\
	&\leq H^* - \pi_{\M} (a^* | s_t) H^* - \big( 1 - \pi_{\M} (a^* | s_t) \big) H_* \\
	&= ( H^* - H_* ) \big( 1 - \pi_{\M} (a^* | s_t) \big).
\end{align*}
	Then, we note that the event of taking a sub-optimal action $a \in \Aspace \setminus \{ a^* \}$ with a policy $\pi_{\M}$ can be modelled by a Bernoulli distribution $\mathcal{B}$ with parameter $\big( 1 - \pi_{\M} (a^* | s_t) \big)$. By combining the equation of the variance of a Bernoulli random variable with Lemma~\ref{thr:law_total_variance} we obtain
	\begin{equation*}
		\Var \big[ \mathcal{B} ( \pi_{\M} (a^*|s_t) ) \big] = \pi_{\M} (a^* | s_t) \big( 1 - \pi_{\M} (a^* | s_t) \big) = \Var_{hs \sim p^{\pi_{\NM}}_t} \big[ \EV \big[ \mathcal{B} ( \pi_{\NM} (a^* | hs_t) ) \big] \big]
	\end{equation*}
	which gives
	\begin{align*}
		\R_{T - t} (\pi_{\M}) &\geq \frac{ H^*- H^*_2 }{\pi_{\M} (a^* | s_t)} \Var_{hs \sim p^{\pi_{\NM}}_t} \big[ \EV \big[ \mathcal{B} ( \pi_{\NM} (a^* | hs_t) ) \big] \big] := \lR_{T - t} (\pi_{\M}) \\
		\R_{T - t} (\pi_{\M}) &\leq \frac{ H^* - H_* }{\pi_{\M} (a^* | s_t)} \Var_{hs \sim p^{\pi_{\NM}}_t} \big[ \EV \big[ \mathcal{B} ( \pi_{\NM} (a^* | hs_t) ) \big] \big] := \uR_{T - t} (\pi_{\M})
	\end{align*}
\end{proof}

\begin{lemma}[Worst-Case CMP]
	For any $t$-step trajectory $h_t \in \Hspace_{[T]}$, taking a sub-optimal action $a \in \Aspace \setminus \{a^*\}$ at step $t$ in the worst-case CMP $\underline{\cmp}$ gives a final entropy
	\begin{equation*}
		H_* = \min_{h \in \Hspace_{T - t}} H(d_{h_t \oplus h} (\cdot)),
	\end{equation*}
	where the minimum is attained by
	\begin{equation*}
		\underline{h}_{T - t} = \big( s_i \in \argmax_{s \in \Sspace} d_{h_t} (s) \big)_{i = t + 1}^{T} \in \argmin_{h \in \Hspace_{T - t}} H(d_{h_t \oplus h} (\cdot)).
	\end{equation*}
	\label{thr:worst_case_cmp}
\end{lemma}
\begin{proof}
	The worst-case CMP $\underline{\cmp}$ is designed such that the agent cannot recover from taking a sub-optimal action $a_t \in \Aspace \setminus \{ a^* \}$ as it is absorbed by a worst-case state given the trajectory $h_t$. A worst-case state is one that maximizes the visitation frequency in $h_t$, \ie $\underline{s} \in \argmax_{s \in \Sspace} d_{h_t} (s)$, so that the visitation frequency becomes increasingly unbalanced. A sub-optimal action at the first step in $\underline{\cmp}$ leads to $T - 1$ visits to the initial state $s_0 \sim \mu$, and the final entropy is zero.
\end{proof}

\begin{lemma}[Best-Case CMP]
	For any $t$-step trajectory $h_t \in \Hspace_{[T]}$, taking a sub-optimal action $a \in \Aspace \setminus \{ a^* \}$ at step $t$ in the best-case CMP $\overline{\cmp}$ gives a final entropy
	\begin{equation*}
		H^*_2 = \max_{h \in \Hspace_{T - t} \setminus \Hspace_{T - t}^*} H(d_{h_t \oplus h} (\cdot)) 
		\quad s. t. \; \Hspace_{T - t}^* = \argmax_{h \in \Hspace_{T - t}} H(d_{h_t \oplus h} (\cdot))
	\end{equation*}
	where the maximum is attained by
	\begin{equation*}
		\overline{h}_{T - t} = s_{2}^* \oplus \big(  \ h_{T - t - 1}^* \in \Hspace_{T - t - 1}^* \big) \in \argmax_{h \in \Hspace_{T - t} \setminus \Hspace^*_{T - t}} H \big( d_{h_t \oplus h} (\cdot) \big),
	\end{equation*}
	in which $s_{2}^*$ is any state that is the second-closest to a uniform entry in $d_{h_t \oplus \overline{h}_{T - t}}$.
	\label{thr:best_case_cmp}
\end{lemma}
\begin{proof}
	The best-case CMP $\overline{\cmp}$ is designed such that taking a sub-optimal action $a \in \Aspace \setminus \{ a^* \}$ at step $t$ minimally decreases the final entropy. Especially, instead of reaching at step $t + 1$ an optimal state $s^*$, \ie a state that maximally balances the state visits of the final trajectory, the agent is drawn to the second-to-optimal state $s_{2}^*$, from which it gets back on track on the optimal trajectory for the remaining steps. Note that visiting $s_{2}^*$ cannot lead to the optimal final entropy, achieved when $s^*$ is visited at step $t + 1$, due to the sub-optimality of action $a$ at step $t$ and Assumption~\ref{ass:unique_action}.
\end{proof}

\paragraph{Instantaneous Regret}
Although the objective~\eqref{eq:finite_samples_entropy} is non-additive across time steps, we can still define a notion of \emph{pseudo-instantaneous regret} by comparing the regret-to-go of two subsequent time steps. In the following, we provide the definition of this expected pseudo-instantaneous regret along with lower and upper bounds to the regret suffered by an optimal Markovian policy.
\begin{definition}[Expected Pseudo-Instantaneous Regret]
	Consider a policy $\pi \in \Pi$ interacting with a CMP over $T - t$ steps starting from the trajectory $h_t$. We define the expected pseudo-instantaneous regret of $\pi$ at step $t$ as $r_t (\pi) := \max \big( 0, \R_{T - t} (\pi, h_t) - \R_{T - t - 1} (\pi, h_{t + 1}) \big)$.
	\label{thr:regret}
\end{definition}
\begin{restatable}[]{corollary}{instantaneousRegretBounds}
    \label{thr:instantaneous_regret_bounds}
	Let $\pi_{\M} \in \Pim$ be a Markovian policy such that $\pi_{\M} \in \argmax_{\pi \in \Pim} \mathcal{E} (\pi)$ on a CMP $\cmp$. Then, for any $h_t \in \Hspace_{[T]}$, it holds $\underline{r}_t (\pi_{\M}) \leq r_t (\pi_{\M}) \leq \overline{r}_t (\pi_{\M})$ such that
	\begin{align*}
		\underline{r}_t (\pi_{\M}) &= \max \Big( 0, \ H^* \big( \V_t (\pi_{\M}) - \V_{t + 1} (\pi_{\M}) \big) - H^*_2 \V_t (\pi_{\M}) + H_{*} \V_{t + 1} (\pi_{\M}) \Big), \\
		\overline{r}_t (\pi_{\M}) &= \max \Big( 0, \ H^* \big( \V_t (\pi_{\M}) - \V_{t + 1} (\pi_{\M}) \big) - H_* \V_t (\pi_{\M}) + H^*_2 \V_{t + 1} (\pi_{\M}) \Big),
	\end{align*}
	where
	\begin{equation*}
		\V_t (\pi_{\M}) := \frac{1}{ \pi_{\M} (a^*| s_t)} \Var_{hs_t \sim p^{\pi_{\NM}}_t} \big[ \EV \big[ \mathcal{B} ( \pi_{\NM} (a^* | hs_t) ) \big] \big].
	\end{equation*}
\end{restatable}
\begin{proof}
	From the Definition~\ref{thr:regret}, we have that $r_t (\pi_{\M}) = \R_{T - t} (\pi_{\M}) - \R_{T - t - 1} (\pi_{\M})$. Recall that
	\begin{align*}
		&\lR_{T - t} (\pi) = \V_{t} (\pi) \big( H^* - H^*_2 \big), 
		&\uR_{T - t} (\pi) = \V_{t} (\pi) \big( H^* - H_* \big),
	\end{align*}
	from Lemma~\ref{thr:regret_bounds}. Then, we can write
	\begin{align*}
		\underline{r}_t (\pi_{\M}) &\geq \lR_{T - t} (\pi_{\M}) - \uR_{T - t - 1} (\pi_{\M}) = H^* \big( \V_t (\pi_{\M}) - \V_{t + 1} (\pi_{\M}) \big) - E^*_2 \V_t (\pi_{\M}) + H_* \V_{t + 1} (\pi_{\M}), 
	\end{align*}
	and
	\begin{align*}
		\overline{r}_t (\pi_{\M}) &\leq \uR_{T - t} (\pi_{\M}) - \lR_{T - t - 1} (\pi_{\M}) = H^* \big( \V_t (\pi_{\M}) - \V_{t + 1} (\pi_{\M}) \big) - H_* \V_t (\pi_{\M}) + H^*_2 \V_{t + 1} (\pi_{\M}).
	\end{align*}
\end{proof}

\subsection{Proofs of Section~\ref{sec:complexity_analysis}}
\label{apx:proofs_complexity_analysis}

\npHardness*

\begin{proof}
To prove the theorem, it is sufficient to show that there exists a problem $\Psi_c \in$ NP-hard so that $\Psi_c \leq_p \Psi_0$. We show this by reducing 3SAT, a well-known NP-complete problem, to $\Psi_0$. To derive the reduction we consider two intermediate problems, namely $\Psi_1$ and $\Psi_2$. Especially, we aim to show that the following chain of reductions hold:
    \begin{equation*}
        \Psi_0 \geq_m \Psi_1 \geq_p \Psi_2 \geq_p \text{3SAT}
    \end{equation*}
First, we define $\Psi_1 $ and we prove that $\Psi_0 \geq_m \Psi_1$. Informally, $\Psi_1$ is the problem of finding a reward-maximizing Markovian policy $\pi_{\M} \in \Pi_{\M}$ w.r.t. the entropy objective~\eqref{eq:finite_samples_entropy} encoded through a reward function in a convenient POMDP $\epomdp$. We can build $\epomdp$ from the CMP $\cmp$ similarly as the extended MDP $\emdp$ (see Section~\ref{sec:preliminaries} and the proof of Lemma~\ref{thr:optimal_deterministic} for details), except that the agent only access the observation space $\widetilde{\Omega}$ instead of the extended state space $\widetilde{\Sspace}$. In particular, we define $\widetilde{\Omega} = \Sspace$ (note that $\Sspace$ is the state space of the original CMP $\cmp$),  $\widetilde{O} (\widetilde{o} | \widetilde{s}) = \widetilde{s}_{-1}$,
and the reward function $\widetilde{R}$ assigns value 0 to all states $\widetilde{s} \in \widetilde{S}$ such that $|\widetilde{s}| \neq T$, otherwise (if $|\widetilde{s}| = T$) the reward corresponds to the entropy value of the state visitation frequences induced by the trajectory codified through $\widetilde{s}$.

Then, the reduction $\Psi_0 \geq_m \Psi_1$ works as follows. We denote as $\mathcal{I}_{\Psi_i}$ the set of possible instances of problem $\Psi_i$. We show that $\Psi_0$ is harder than $\Psi_1$ by defining the polynomial-time functions $\psi$ and $\phi$ such that any instance of $\Psi_1$ can be rewritten through $\psi$ as an instance of $\Psi_0$, and a solution $\pi_{\NM}^* \in \Pinm$ for $\Psi_0$ can be converted through $\phi$ into a solution $\pi_{\M}^* \in \Pim$ for the original instance of $\Psi_1$.
    \[ \begin{tikzcd}
    \mathcal{I}_{\Psi_1} \arrow{r}{\psi} & \mathcal{I}_{\Psi_0} \arrow{d} \\%
    \pi_{\M}^* & \arrow{l}{\phi} \pi_{\NM}^*
    \end{tikzcd}
    \]
The function $\psi$ sets $\Sspace = \widetilde{\Omega}$ and derives the transition model of $\cmp$ from the one of $\epomdp$, while $\phi$ converts the optimal solution of $\Psi_0$ by computing 
\begin{equation}
    \pi_{\M}^* (a | o, t) = \sum_{ho \in \Hspace_o} p^{\pi_{\NM}^{*}}_t (ho) \pi_{\NM}^* (a|ho) 
\end{equation} 
where $\Hspace_o$ stands for the set of histories $h \in \Hspace_{t}$ ending in the observation $o \in \Omega$.
Thus, we have that $\Psi_0 \geq_m \Psi_1$ holds. We now define $\Psi_2$ as the policy existence problem w.r.t. the problem statement of $\Psi_1$. Hence, $\Psi_2$ is the problem of determining whether the value of a reward-maximizing Markovian policy $\pi_{\M}^* \in \argmax_{\pi \in \Pim} \mathcal{J}_{\epomdp} (\pi)$ is greater than 0. Since computing an optimal policy in POMDPs is in general harder than the relative policy existence problem ~\citep[][Section 3]{lusena2001complexity}, we have that $\Psi_1 \geq_p \Psi_2$.

For the last reduction, \ie $\Psi_2 \geq_p \text{3SAT}$, we extend the proof of Theorem 4.13 in \cite{mundhenk2000complexity}, which states that the policy existence problem for POMDPs is $\NP$-complete. In particular, we show that this holds within the restricted class of POMDPs defined in $\Psi_1$.

The restrictions on the POMDPs class are the following:
\begin{enumerate}
    \item The reward function $R(s)\geq 0$ only in the subset of states reachable in T steps, otherwise $R(s)=0$;
    \item $|\widetilde{\Sspace}|=\widetilde{S}=|\widetilde{\Omega}|^T$.
\end{enumerate}
Both limitations can be overcome in the following ways:
\begin{enumerate}
    \item It suffices to add states with deterministic transitions so that $T=m\cdot n$ can be defined a priori, where T is the number of steps needed to reach the state with positive reward through every possible path. Here $m$ is the number of clauses, and $n$ is the number of variables in the 3SAT instance, as defined in \citep{mundhenk2000complexity};
    \item The POMDPs class defined by $\Psi_1$ is such that $\widetilde{S}=|\widetilde{\Omega}|^T$. Noticing that the set of observations corresponds with the set of variables and that from the previous point $T=m\cdot n$, we have that $|\widetilde{\Omega}|^T = n^{m\cdot n}$, while the POMDPs class used by the proof hereinabove has $\widetilde{S}=m\cdot n^2$. Notice that $n \geq 2$ and $m \geq 1$ implies that $n^{m\cdot n} \geq m\cdot n^2$. Moreover, notice that every instance of 3SAT has $m\geq 1$ and $n\geq 3$. Hence, to extend the proof to the POMDPs class defined by $\Psi_1$ it suffices to add a set of states $\widetilde{S}_p$ s.t. $R(s)=0 \; \forall s\in \widetilde{S}_p$.
\end{enumerate}
Since the chain $\Psi_0 \geq_m \Psi_1 \geq_p \Psi_2 \geq_p \text{3SAT}$ holds, we have that $\Psi_0 \geq_p \text{3SAT}$. Moreover, since $ \text{3SAT} \in \NP$-complete, we can conclude that $\Psi_0$ is $\NP$-hard.
\end{proof}


\end{document}